\documentclass[11pt]{article}

\usepackage[margin=1in]{geometry}
\usepackage[authoryear,round]{natbib}
\setcitestyle{citesep={;},aysep={,},yysep={;}}
\usepackage{amsmath,amssymb,amsthm,mathtools}
\usepackage{etoolbox}
\usepackage{needspace}
\usepackage{pgfplots}
\pgfplotsset{compat=1.18}
\usepgfplotslibrary{groupplots}
\usepackage{xcolor}
\definecolor{linkblue}{RGB}{38,73,115}
\usepackage[hyperfootnotes=false]{hyperref}
\usepackage{url}
\hypersetup{
  colorlinks=true,
  linkcolor=linkblue,
  citecolor=linkblue,
  urlcolor=linkblue,
  pdftitle={Bayesian Intelligence and Its Implications},
  pdfauthor={Alex Smolin and Bryan Wilder}
}

\newcommand*{\appref}[1]{\hyperref[#1]{Appendix~\ref*{#1}}}

\newtheoremstyle{mainresult}
  {8pt plus 2pt minus 2pt}
  {4pt plus 1pt minus 1pt}
  {\itshape}
  {}
  {\bfseries}
  {.}
  {0.5em}
  {}

\newenvironment{compactproof}
  {\begin{proof}}
  {\end{proof}}

\theoremstyle{mainresult}
\newtheorem{proposition}{Proposition}

\newtheorem{corollary}{Corollary}

\newtheorem{lemma}{Lemma}

\theoremstyle{definition}

\newtheorem{assumption}{Assumption}

\newtheorem{example}{Example}

\newtheorem*{examplecontinued}{\autoref*{ex:running_example}}
\newcommand{\exampleqed}{%
  \begingroup
  \renewcommand{\qedsymbol}{\ensuremath{\blacklozenge}}%
  \qed
  \endgroup
}
\AtBeginEnvironment{example}{\pushQED{\exampleqed}}
\AtEndEnvironment{example}{\popQED}
\AtBeginEnvironment{examplecontinued}{\pushQED{\exampleqed}}
\AtEndEnvironment{examplecontinued}{\popQED}
\theoremstyle{definition}
\newtheorem{remark}{Remark}
\newcommand{\remarkqed}{%
  \begingroup
  \renewcommand{\qedsymbol}{\ensuremath{\lozenge}}%
  \qed
  \endgroup
}
\AtBeginEnvironment{remark}{\pushQED{\remarkqed}}
\AtEndEnvironment{remark}{\popQED}
\theoremstyle{remark}

\theoremstyle{mainresult}
\newtheorem{theorem}{Theorem}
\newcounter{blackwellcondition}
\renewcommand{\theblackwellcondition}{\roman{blackwellcondition}}

\DeclareMathOperator{\supp}{supp}

\newcommand{\EE}{\mathbb E}

\newcommand{\RR}{\mathbb R}
\newcommand{\reals}{\mathbb R}

\newcommand{\Del}{\Delta}

\newcommand{\Y}{\mathcal{Y}}

\title{Bayesian Intelligence from the Outside}
\author{
	Alex Smolin\\
    Department of Economics\\
Toulouse School of Economics\\
alex.smolin@tse-fr.eu
	\and
Bryan Wilder\\
	Machine Learning Department\\ Carnegie Mellon University\\
	bwilder@cmu.edu
 }

\date{}
\usepackage{titling}
\pretitle{\begin{center}\LARGE\bfseries}
\posttitle{\par\end{center}\vskip 0.25em}
\predate{}
\postdate{}

\begin{document}
\maketitle

\begin{abstract}
Inferring intelligence from observable behavior is a foundational challenge in artificial intelligence. We develop a theory of Bayesian intelligence for agents such as language models. Each prompt induces a possibly imperfect internal experiment; the agent updates a full-support prior by Bayes' rule and faithfully reports its posterior over the possible answers to the question. Repetitions draw fresh, independent outcomes from the same unobserved experiment at one fixed state. We show that the agent's behavior admits this explanation if and only if its reports are not fully contradictory, i.e., some state remains possible under every report across all prompts. Report frequencies and the sizes of positive probabilities impose no further restrictions. We further propose and characterize the behavioral implications of an intelligence order that makes the behavior of two agents consistent with one agent having access to a more informative experiment: there should exist a coupling of report distributions such that the more informative agent's report excludes every answer excluded by its counterpart. Finally, we show the difficulty of aggregating coarse reports from intelligent agents: unless the agent reports a belief about the complete state of the world, the optimal aggregation can assign arbitrary weights to states that have not been excluded. These results provide a basis for understanding when agents' behavior is intelligent and highlight the difficulty of rejecting Bayesian rationality.
\end{abstract}

\section{Introduction}

As systems based on large language models (LLMs) are integrated deeper into decision-making and economic contexts \citep{DengEtAl2024,DuettingEtAl2024,DuettingEtAl2026,FishEtAl2026}, understanding the mechanism through which LLMs reason and whether this process is internally consistent or rationalizable becomes increasingly important. Internal consistency is a necessary component of valid reasoning and determines whether a system can be modeled from the outside as a rational agent. For example, classic representation theorems provide axiomatic foundations for beliefs as drivers of rational behavior: when behavior is sufficiently consistent, it can be represented by a utility function and subjective probabilities over the relevant uncertainty, such that an agent behaves ``as if'' maximizing expected utility given those beliefs \citep{Savage1954,AnscombeAumann1963}. Bayes' rule \citep{Bayes1763} further prescribes how beliefs should be updated in response to external information. A growing literature tests whether LLMs update in a Bayes-consistent manner to evidence supplied in the prompt \citep{ChenEtAl2026,ZhuGriffiths2024,andrews2026dutchbookslanguagemodels}, or whether they express preferences that are sufficiently internally coherent to be rationalizable \citep{mazeika2026utility}.   

In addition to externally provided evidence, sophisticated behavior from LLMs is now underpinned by substantial inference-time reasoning. Our goal is to develop an analogous framework for testing the internal consistency of this process: we model inference-time computation as supplying ``internal'' information that updates the agent's beliefs. Testing whether the update process is rationalizable in Bayesian terms allows us to separate quality of the information acquired (which may be variable, depending e.g., on the model's domain-specific or factual knowledge) from the correctness of updating.

The challenge is that the information driving LLMs' answers is often opaque: a large body of work questions whether the text of LLM chain of thought provides faithful explanations of the output \citep{lanham2023measuring,chen2025reasoning}, or whether reasoning always plays a causal role how LLMs solve problems \citep{boppana2026reasoning,palod2025performative}. We provide a complementary perspective on such problems by asking how much can be tested within a black-box framework. Formally, we introduce a model based on the classical Blackwell framework of statistical experiments where reasoning reveals \textit{information} to the LLM, which the LLM is then able to condition on in producing its final response. We do not make any assumptions about whether this intermediate information is ``faithful" when viewed as natural language. Instead, we ask just whether the distribution of final reports from the LLM under different prompts or repeated resampling of the reasoning process is internally rationalizable, i.e., Bayes-consistent. Building on the same framework, we also show how to test whether on agent's reasoning reveals strictly more information than another, and how repeated reports from a LLM should be combined.

Our framework separates the correctness of reasoning, as judged by consistency with external ground truth, from its internal consistency by allowing for the possibility experiment can be imperfect, so the agent can place little posterior probability on the truth and give an incorrect answer. The experiment may also depend arbitrarily on the prompt: our formal model allows the possibility that a change in framing can change the reasoning path. The model does require correct specification in two senses: the prior gives the true state positive probability, and the actual signal distribution is the one used in Bayes' rule. Thus, conditional on its signal, the agent updates without error.

Our first result, \autoref{thm:general_full_belief}, characterizes the collections of \emph{Bayes-plausible} report distributions, i.e., those consistent with this model. A collection is Bayes-plausible if and only if some single complete state remains possible under every report for every prompt. In other words, zero-probability claims---``hard exclusions''---must be jointly consistent; conditional report frequencies and the numerical values of positive probabilities are otherwise unrestricted. In particular, any collection of distributions containing only reports that assign positive probability to every answer is rationalizable. No martingale restriction follows from sampling reports at a fixed state. This rationalization establishes compatibility with Bayesian reasoning, without identifying how the agent actually reasons.

We next compare the information acquired by two agents. For a given prompt, one report distribution \emph{plausibly Blackwell dominates} another if both can arise at a common prior and true state, with a more informative experiment behind the first. When each possible answer is compatible with at least two complete states, \autoref{prop:blackwell} shows that this holds exactly when all reports share a possible true answer and the reports can be paired so that each more informative report rules out every answer ruled out by its less informative counterpart, and possibly more. Thus, the observable restrictions on a plausible intelligence order concern coherent hard exclusions. With only interior reports, either order is compatible with the observations, regardless of confidence or variability.

Finally, we ask what repeated reports reveal about their Bayesian aggregate when the reasoning experiment is unknown. By \autoref{prop:bayes_plausible_aggregation}, with at least two distinct reports and at least two complete states per answer, the aggregate must rule out every answer excluded by a report, but can assign any weights to the other answers. If one report rules out an answer and another does not, their arithmetic average violates this restriction. Affine log-odds pooling respects the restriction, and the optimal rule in the Gaussian intelligence model takes this form. When reports give beliefs about the complete state, a prior-corrected sum of their log odds determines the Bayesian aggregate. This requires knowing the prior, but not the experiment. Coarse reports leave out distinctions between states that are needed for this calculation.

\subsection{Related Literature}

\paragraph{Bayesian Reasoning in Language Models.}
One strand of work models in-context learning as implicit Bayesian inference under a specified generative model of pretraining data and prompts \citep{XieEtAl2022}. We ask which report distributions can be explained by Bayesian updating when the experiment induced by each prompt is unknown and unrestricted. The aim is to establish what can be tested from these reports; the model does not explain how pretraining produces Bayesian inference.

Many existing evaluations test several forms of logical or probabilistic consistency. BeliefBank treats categorical answers about propositions as a single, stable set of beliefs \citep{KassnerEtAl2021}. Other studies check probability identities across prompts \citep{ZhuGriffiths2024,ImranEtAl2025} or compare beliefs elicited later with Bayesian updates calculated from elicited priors and likelihoods \citep{ChenEtAl2026}. Martingale tests examine in-context predictions or reasoning histories \citep{FalckWangHolmes2024,HeEtAl2025}, while evidence-based tests vary the informativeness and reliability of supplied evidence \citep{KimKimThorne2025}. These evaluations maintain implicit and explicit assumptions about what responses represent and how information is related across prompts. We derive the implications of Bayesian updating when those information relationships are unrestricted. A categorical answer need not express certainty, and different prompts can induce different reasoning experiments.

Other benchmarks compare probabilistic judgments with reference distributions. CogBench's urn task supplies priors and likelihoods and estimates how models weigh them \citep{CodaFornoEtAl2024}; QUITE tests numerical probabilities and verbal uncertainty given verbalized Bayesian networks and evidence \citep{SchraderEtAl2024}; and \citet{ParuchuriEtAl2024} evaluate percentile estimation, sampling, and probability calculations. These evaluations measure how well models solve particular inference tasks. We ask whether their reports can be explained by Bayesian updating under an unknown internal experiment that may provide imperfect information (see discussion in \autoref{sec:discussion} on a possible combination of these approaches).

\paragraph{Bayesian Learning and Its Empirical Content}
One strand of work studies which distributions of posterior beliefs can arise before the state is known. The splitting lemma of \citet{AumannMaschler1995} and the persuasion formulation of \citet{KamenicaGentzkow2011} show that a finite distribution of posteriors can be induced by a signal if and only if its mean equals the prior. \citet{ArieliBabichenko2024} extend this characterization to a homogeneous population. These results average over both states and signals under the prior-predictive distribution. By contrast, \citet{ArieliEtAl2020Identifiable} characterize forecast distributions conditional on a realized state and ask when they identify that state despite an unknown information structure. Similarly, \citet{DovalSmolin2024} study welfare conditional on each realized type, retaining information hidden by ex ante averaging. We condition on one state, but study the joint feasibility of coarse reports from multiple prompt-dependent experiments, as well as compatible information orders and aggregates. Because our observations do not average across states, \autoref{thm:general_full_belief} imposes much weaker restrictions.

A second strand asks what Bayesian updating implies when the analyst does not fully know the agent's information structure. \citet{ShmayaYariv2016} obtain an ``anything goes'' result: any assignment of posterior beliefs to signal histories can be explained by Bayesian updating when subjects' conjectures about the joint distribution of the state, signals, and number of signals revealed are unrestricted. Subjects can treat the experimenter's choice of how much information to reveal as evidence about the state, beyond the signals' content. Protocols that restrict these conjectures recover testable implications. \citet{Molavi2026Tests} shows that martingale restrictions can also fail when the agent's subjective signal distribution differs from the objective distribution. \citet{Molavi2026Misspecification} further shows that unrestricted misspecification on a finite state space leaves only support restrictions. In our setting, the agent's signal model is correctly specified, but the prompt-dependent internal experiments are unobserved. Repeated reports are sampled at one fixed state and must all leave it possible; unobserved signal probabilities at other states supply the remaining freedom for rationalization.

\paragraph{Behavioral and Experimental Economics} Our theory can also be applied to human intelligence and therefore speaks to behavioral and experimental economics. Like our model, the tracing procedure of \citet{HarsanyiSelten1988} formalizes an internal reasoning process that starts from prior beliefs. It instead adjusts beliefs about others' strategies through iterated best-response reasoning until plans and expectations coincide at an equilibrium; our agent conditions on a prompt-induced signal about the state and may remain uncertain. Rational inattention  endogenizes the information structure: agents choose what and how much to learn under a finite Shannon capacity \citep{Sims2003} or a Shannon information cost that yields generalized multinomial-logit choice \citep{MatejkaMcKay2015}. We instead allow an arbitrary prompt-dependent experiment and ask what Bayesian updating alone implies, without assuming optimal information acquisition or a particular cost. Recently, \citet{CaplinMartinMarx2025} apply cognitive-economic revealed-preference tests to a trained image classifier, varying its training loss and rationalizing its predictions through costly learning. Their object is learning across training environments; ours is posterior formation during inference by a fixed trained system. We note that AI systems are especially useful for this inquiry because an identical system and configuration can be queried repeatedly in a controlled environment. Human subjects differ and remember prior exposure, which confounds repeated measurement and requires more restrictive designs.

\section{Model}
An agent responds to prompts from an observer. We propose a Bayesian theory of how it acquires information, updates its beliefs, and responds.

Formally, we let $\omega\in\Omega$ denote a state of the world, where $\Omega$ is finite and $|\Omega|=N$. The state encompasses \emph{all} uncertainty about the world relevant to the agent's answers, so we expect $N$ to be very large. The agent starts with a full-support \emph{prior}
\begin{align*}
\mu_0\in\operatorname{int}\Del(\Omega).
\end{align*}
This prior is the starting point for the agent's computation and can be interpreted as summarizing its knowledge about the world in the absence of any additional computation.

The observer sends a prompt $p$ from a given vocabulary $P$. Each prompt triggers a reasoning process that leads the agent to revise its belief via computation that includes the analysis of the prompt, the drawing on its internal knowledge, the chain-of-thought procedure \citep{WeiEtAl2022}, and possibly a  retrieval of additional external information via tools. The agent then outputs a report $r$ from a given vocabulary $R$.

We model this process as a Bayesian learning about the state of world followed by a strategic response. First, the agent gets access to a prompt-dependent statistical experiment
\begin{align}
E_p=(S_p,\pi_p),
\qquad
\pi_p:\Omega\to\Del(S_p),
\label{eq:experiment}
\end{align}
with finite $S_p$. The agent observes a signal $s\in S_p$ and, building on its prior, forms the \emph{posterior} by Bayes' rule,
\begin{align*}
\mu_p(s)(\omega)
=
\frac{\mu_0(\omega)\pi_p(s\mid\omega)}
{\sum_{\omega'\in\Omega}\mu_0(\omega')\pi_p(s\mid\omega')},
\end{align*}
whenever $s$ has positive ex ante probability. Second, the agent responds according to a \emph{strategy} that maps the prompt and posterior to a distribution over responses:
\begin{align*}
\sigma:P\times\Del(\Omega)\to\Del(R).
\end{align*}

We focus on a family of prompts that request the agent's posterior $\nu_p\in\Delta(\Y_p)$ over the possible answers to a question. We represent the correct answer by a state-measurable surjective random variable $Y_p:\Omega\to\Y_p$ and define
\begin{align}
\nu_p(\mu)(y)=\sum_{\omega\in\Omega:Y_p(\omega)=y}\mu(\omega),\quad \forall\,y\in\Y_p.
\end{align}
This focus is without loss of generality as we discuss in  \appref{app:partition_reduction}.

For $y\in\Y$, let $\Omega_y:=\{\omega'\in\Omega:Y(\omega')=y\}$ denote the set of states compatible with $y$. In practice, because the state space is enormous and cannot be fully communicated, the partition of $\Omega$ induced by $Y$ should be expected to be very coarse. Some of our results impose the weak coarseness assumption $|\Omega_y|\geq2$ for every $y\in\Y$; we state this assumption explicitly whenever it is needed.

We call the agent \emph{obedient} at prompt $p$ if it reports the requested function with probability one:
\begin{align*}
\sigma(p,\mu)=\delta(\nu_p(\mu)),\quad\forall\mu\in\Delta(\Omega).
\end{align*} 
Here $\delta(x)$ denotes the point mass at $x$.

\begin{assumption}[Obedience]\label{ass:obedience} The agent is obedient at every prompt. 
\end{assumption}
The obedience assumption allows us to isolate the restrictions imposed by Bayesian reasoning from errors or strategic distortions in the reporting stage.
It strengthens our findings that Bayesian reasoning is difficult to reject.

The observer can submit any prompt to the agent any number of times. The state remains fixed across repetitions, and the experiment associated with the same prompt remains unchanged. As such, the agent's inner workings are stationary and outside the observer's control. However, the reasoning process begins afresh in each repetition. That is, conditional on the state, the signals are independent and identically distributed:
\begin{align*}
s_1(p),s_2(p),\dots\mid\omega
\stackrel{\mathrm{i.i.d.}}{\sim}
\pi_p(\cdot\mid\omega).
\end{align*}

If the conditional signal distribution $\pi_p(\cdot\mid\omega)$ is degenerate, i.e., concentrated on a single signal, an obedient agent gives the same answer in every repetition. Otherwise, the agent may form different posteriors and therefore give different reports across repetitions. Thus, Bayesian reasoning is fully compatible with stochastic responses from the agent.

Denote the report distribution induced by prompt $p$ and the true state by
\begin{align*}
\tau_p\in\Del(\Del(\Y_p)).
\end{align*}
Since the signal space is finite, $M_p\triangleq\supp\tau_p$ is finite. By the strong law of large numbers, the empirical distribution of reports converges almost surely to $\tau_p$. An idealized observer who can repeat the prompt indefinitely therefore observes this distribution exactly. Accordingly, we pose our formal results in the following sections in terms of the report distributions.

\begin{remark}[Correct Specification]
We note that the model implies two forms of correct specification. First, it has a grain of truth: because the prior has full support, the true state receives positive prior probability. Second, the signal distribution is correctly specified: for each prompt $p$, conditional on each state $\omega$, the signal is drawn from the same distribution $\pi_p(\cdot\mid\omega)$ that the agent uses in Bayes' rule.
\end{remark}

\begin{remark}[Prompt Dependence] The model does not impose any restrictions on how the agent's reasoning, $E_p$, depends on the prompt $p$. Even a small change in the prompt's framing can change responses (e.g., \citet{SclarEtAl2024}). Thus, we isolate the implications of Bayesian reasoning from those of frame invariance. \autoref{sec:discussion} discusses additional restrictions linking prompts and their implications for evaluation.
\end{remark}

\section{Plausible Report Distributions}\label{sec:distributions}
We first ask when repeated reports can rule out the Bayesian intelligence model. Consider responses to a finite collection of prompts $P$. Although the prompt-specific experiments may differ, they must share the same prior and true state. We say that a collection of observed report distributions $\{\tau_p\}_{p\in P}$ is \emph{Bayes-plausible} if there exist a prior $\mu_0\in\operatorname{int}\Delta(\Omega)$, a state $\omega\in\Omega$, and finite experiments $\{E_p\}_{p\in P}=\{(S_p,\pi_p)\}_{p\in P}$ such that, when the true state is $\omega$, the distribution of the agent's posterior report after prompt $p$ is $\tau_p$ for every $p\in P$.

\Needspace{9\baselineskip}
\begin{theorem}[Bayes-Plausibility]
\label{thm:general_full_belief}
The collection of report distributions $\{\tau_p\}_{p\in P}$ is Bayes-plausible if and only if
\begin{align}
\bigcap_{p\in P}
Y_p^{-1}
\left(
\bigcap_{\nu\in\supp\tau_p}\supp\nu
\right)
\neq
\varnothing.
\label{eq:general_compatibility_multiple}
\end{align}
\end{theorem}
\begin{compactproof}
See \appref{app:proof:general_full_belief}.
\end{compactproof}

At the true state, a report can only come from a signal with positive probability there. With a full-support prior, Bayes' rule then gives the true answer positive posterior probability. This proves necessity. For sufficiency, the case of a single state is immediate. Otherwise, choose any state consistent with all reports and assign it a sufficiently small positive probability under a common full-support prior. For each prompt, we can extend its answer-level reports to full-state posteriors while preserving their reported marginals. The prior can then be chosen small enough at the candidate state to be able to construct a prompt-specific experiment inducing each distribution at the same state and under the same prior.

The restriction in \autoref{thm:general_full_belief} is therefore a consistency requirement on hard exclusions across prompts and reports. Positive probabilities and their frequencies impose no further restrictions. For repetitions of one prompt, the condition reduces to $\bigcap_{\nu\in\supp\tau}\supp\nu\neq\varnothing$.

\begin{remark}[Adaptive prompting]
This conclusion extends directly to adaptive prompting. Suppose the observer  dynamically chooses a sequence of prompts over a finite number of rounds. We can treat each history together with its next prompt as a distinct prompt. By \autoref{thm:general_full_belief}, the whole response tree is Bayes-plausible if and only if a single state $\omega$ is compatible with every report at every reachable node. In particular, any such tree of full-support reports is Bayes-plausible. For example, no martingale properties should be expected.
\end{remark}

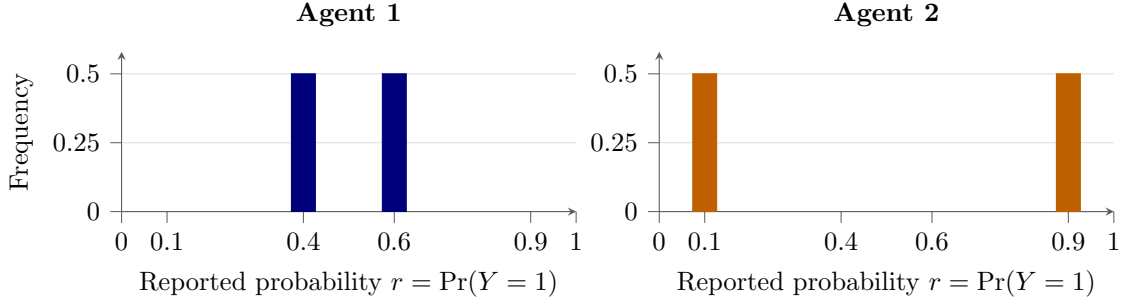
\begin{figure}[!t]
\centering
\begin{tikzpicture}
\begin{groupplot}[
  group style={group size=2 by 1,horizontal sep=1.1cm},
  width=0.46\linewidth,height=3.7cm,
  xmin=0,xmax=1,ymin=0,ymax=0.58,
  xtick={0,0.1,0.4,0.6,0.9,1},
  xticklabels={$0$,$0.1$,$0.4$,$0.6$,$0.9$,$1$},
  ytick={0,0.25,0.5},yticklabels={$0$,$0.25$,$0.5$},
  axis lines=left,axis line style={black!65},
  tick style={black!65},tick align=outside,
  tick label style={font=\small},
  label style={font=\small},title style={font=\small\bfseries},
  xlabel={Reported probability $r=\Pr(Y=1)$},
  ylabel={Frequency},
  ymajorgrids=true,grid style={black!12},
  ybar,/pgf/bar width=9pt,
  enlarge x limits=false,
]
\nextgroupplot[title={Agent 1}]
\addplot[fill=blue!48!black,draw=blue!48!black] coordinates {(0.4,0.5) (0.6,0.5)};
\nextgroupplot[title={Agent 2},ylabel={}]
\addplot[fill=orange!75!black,draw=orange!75!black] coordinates {(0.1,0.5) (0.9,0.5)};
\end{groupplot}
\end{tikzpicture}
\caption{Report distributions for the running example. Each displayed report occurs with probability $1/2$. Agent 1's reports are less variable and less confident than agent 2's.}
\label{fig:running_example}
\end{figure}

\begin{example}\label{ex:running_example}
Let $r=\Pr(Y=1)$ denote a report about a binary question, e.g., whether a patient with a given medical history has chronic kidney disease. \autoref{fig:running_example} shows two stylized distributions of repeated reports, $\tau_1=\tfrac12\delta_{0.4}+\tfrac12\delta_{0.6}$ and $\tau_2=\tfrac12\delta_{0.1}+\tfrac12\delta_{0.9}$. Both have mean $1/2$, but agent 2's reports are more confident and less consistent. By \autoref{thm:general_full_belief}, both distributions are Bayes-plausible: even the sharp disagreement between agent 2's reports does not rule out Bayesian reasoning. One supporting construction is $\Omega=\{0,1\}$ with $Y(\omega)=\omega$, common prior $\mu_0=(5/6,1/6)$, and true state $1$. For rows ordered as $(0,1)$, the conditional signal probabilities
\begin{align}
\pi_1=\begin{pmatrix}
3/20&1/15&47/60\\
1/2&1/2&0
\end{pmatrix},
\qquad
\pi_2=\begin{pmatrix}
9/10&1/90&4/45\\
1/2&1/2&0
\end{pmatrix}
\label{eq:running_binary_experiments}
\end{align}
induce reports $(0.4,0.6,0)$ and $(0.1,0.9,0)$, respectively. The third signal, although it never occurs at the true state, is required for construction at this prior.

Allowing the prior and experiments to vary, replacing agent 2's lower report by any $\varepsilon>0$ sufficiently close to zero preserves Bayes-plausibility. Setting it to exactly zero instead forces the candidate true answer to be $0$; moving the upper report to $1$ as well makes the reports jointly incompatible. The logical implication changes at zero, however small the preceding positive probability was.
\end{example}

\section{Plausible Intelligence Order}\label{sec:order}
Compatibility with Bayesian reasoning does not by itself rank agents. We next ask when their report distributions are compatible with one agent acquiring more informative reasoning outcomes than another. Blackwell informativeness provides a task-specific notion of an intelligence order within our model: it compares the information available for decisions after reasoning, holding prior knowledge fixed.

Formally, we consider two agents and fix a prompt $p$ representing a specific task, so we omit prompt-dependence from notation. The agents have the same prior, but agent $i\in\{1,2\}$  has access to the experiment $E_i=(S_i,\pi_i)$. Denote by $\tau_i$ the report distribution of agent $i$ and let $M_i:=\supp \tau_i$.

As standard, experiment $E_2$ is \emph{Blackwell more informative} than $E_1$ if $E_1$ can be generated by first observing the  signal from $E_2$ and then garbling it, i.e., if there exists a stochastic matrix, $G:S_2\to\Del(S_1)$, such that for every state $\omega'\in\Omega$ and every signal $s_1\in S_1$,
\begin{align*}
\pi_1(s_1\mid\omega')
=
\sum_{s_2\in S_2}\pi_2(s_2\mid\omega')G(s_1\mid s_2).
\end{align*} 
By Blackwell's theorem (\cite{Blackwell1953}), having access to a more informative experiment weakly improves optimal expected payoffs, evaluated under the prior, in every decision problem.

We say that $\tau_2$ \emph{plausibly Blackwell dominates $\tau_1$} if there exist a prior $\mu_0\in\operatorname{int}\Del(\Omega)$, a true state $\omega\in\Omega$, and experiments $E_1,E_2$ such that $E_2$ Blackwell dominates $E_1$ and the report distribution under $E_i$, conditional on $\omega$, is $\tau_i$ for $i\in\{1,2\}$.

With a common prior, the posterior distribution generated by the more informative experiment, averaged across states, is a mean-preserving spread of the distribution generated by the less informative experiment. Repeated prompting instead samples reports at one fixed state. The characterization below shows how much this change in the sampling scheme weakens the observable restrictions.

\Needspace{9\baselineskip}
\begin{theorem}[Plausible Intelligence Order]
\label{prop:blackwell}
Suppose $|\Omega_y|\geq2$ for every $y\in\Y$.
The distribution $\tau_2$ plausibly Blackwell dominates $\tau_1$ if and only if \setcounter{blackwellcondition}{0}\refstepcounter{blackwellcondition}\label{eq:blackwell_common_truth}\textup{(\theblackwellcondition)} there exists $y\in\bigcap_{\nu\in M_1\cup M_2}\supp\nu$, and \refstepcounter{blackwellcondition}\label{eq:blackwell_support_coupling}\textup{(\theblackwellcondition)} there exists a coupling $\gamma$  of $\tau_2$ and $\tau_1$ satisfying $\gamma(\nu_2,\nu_1)>0\Longrightarrow\supp\nu_2\subseteq\supp\nu_1$.\footnote{A \emph{coupling} of $\tau_2$ and $\tau_1$ is a distribution $\gamma\in\Del(M_2\times M_1)$ with respective marginals $\tau_2$ and $\tau_1$.}
\end{theorem}

\begin{compactproof}
See \appref{app:proof:blackwell}.
\end{compactproof}

Intuitively, every report observed at the true state must assign positive probability to the same true value $y$, which gives \eqref{eq:blackwell_common_truth}. Along any report pair with positive coupling mass, a less informative report may restore values that its more informative precursor rules out, but it cannot rule out a value that the precursor assigns positive probability; this is the support inclusion in \eqref{eq:blackwell_support_coupling}. Conversely, because the true answer cell contains another state, that state provides enough probability slack to implement any report values consistent with these two support requirements. (If an answer cell is a singleton, there is an additional  condition presented in  \autoref{lem:blackwell_singleton} in the Appendix.) Thus, the plausibility of the intelligence order can be checked as a finite linear feasibility problem on $M_2\times M_1$, rather than by searching over state-contingent experiments and garblings.

The coupling requirement in \autoref{prop:blackwell} admits an equivalent geometric representation. Associate each report $\nu$ with its minimal simplex face $\Delta(\supp\nu)$. Every inclusion-upward-closed collection of faces must then receive weakly more mass under $\tau_1$ than under $\tau_2$.  Equivalently, $\tau_1$'s induced face distribution first-order stochastically dominates $\tau_2$'s under face inclusion. This representation makes it especially clear that only supports of reported beliefs matter for plausible Blackwell dominance.

Conditions for a plausible intelligence order become especially simple when the reported belief is about a binary event, i.e., $|\Y|=2$. In this case the extreme beliefs are $\nu=(1,0),(0,1)$.

\Needspace{6\baselineskip}
\begin{corollary}[Intelligence Order. Binary Event.]
\label{cor:blackwell_binary}
If $|\Y|=2$ and $|\Omega_y|\geq2$ for every $y\in\Y$, then the distribution $\tau_2$ plausibly Blackwell dominates $\tau_1$ if and only if (i) $\supp\tau_1\cup\supp\tau_2$ contains at most one extreme belief and (ii) $\tau_2(\nu)\geq \tau_1(\nu)$ for both extreme beliefs $\nu$.
\end{corollary}

\begin{compactproof}
See \appref{app:proof:blackwell_binary}.
\end{compactproof}

\begin{remark}[Different Priors]
In defining the plausible intelligence order, we required the agents to share a prior to place them on equal footing. However, because the order effectively depends only on the supports of the reports, exactly the same conditions apply even when the agents are allowed to have different priors, e.g., if they are trained on different datasets (see \appref{app:proof:different_priors}).
\end{remark}

\begin{remark}[Evaluation Caution]
Plausible dominance means that some compatible experiments have the proposed order; it does not establish that the agents' actual experiments do. Its interpretation also relies on obedience (\autoref{ass:obedience}). For example, under the conditions of \autoref{cor:blackwell_binary}, a disobedient agent that mechanically reports the same extreme belief has a distribution that plausibly Blackwell dominates any distribution containing only interior reports. This is a compatible explanation of its outputs, even though the reporting rule itself acquires no information. The criterion therefore restricts possible intelligence orders without providing a sufficient test of actual superiority.
\end{remark}

\begin{examplecontinued}[ctd.]
On the two-state space of \autoref{sec:distributions}, no Blackwell ordering can generate the two histograms in \autoref{fig:running_example}, under any common prior (see \appref{app:running_example_order}). Allowing unreported distinctions within an answer makes either order possible. For example, let $\Omega=\{0,1,2\}$, with $Y(0)=0$ and $Y(1)=Y(2)=1$, true state $1$, and the same full-state prior $\mu_0=(100,1,19)/120$ for both agents. This preserves $\Pr(Y=1)=1/6$. States $1$ and $2$ represent circumstances that can affect reasoning without changing the answer to the question. Let $E_2$ have four signals $s_1,\ldots,s_4$ and conditional signal probabilities (rows ordered as $(0,1,2)$)
\begin{align}
\pi_2=\begin{pmatrix}
9/200&1/200&3/20&4/5\\
1/2&1/2&0&0\\
0&4/19&15/19&0
\end{pmatrix}.
\label{eq:running_three_state_experiment}
\end{align}
The four signals induce reports $(0.1,0.9,0.5,0)$, respectively. Construct $E_1$ by garbling $s_1$ to $t_1$, $s_2$ to $t_2$, and $s_4$ to $t_3$; after $s_3$, choose $t_1$ or $t_2$ with equal probability. The resulting signals $(t_1,t_2,t_3)$ induce reports $(0.4,0.6,0)$. At true state $1$, only the first two signals of each experiment occur, each with probability $1/2$, reproducing the plotted histograms. Thus agent 2's less consistent report distribution is compatible with a strictly more informative experiment. \appref{app:running_example_order} constructs the opposite strict order under the same full-state prior. The histograms therefore cannot distinguish these opposite information rankings.

By the proof of \autoref{cor:blackwell_binary}, either ranking remains plausible as agent 2's upper report approaches $1$, allowing the common prior and experiments to vary. At $1$, however, only agent 2 can plausibly be more informative: it sometimes rules out $Y=0$, while agent 1 excludes neither answer.
\end{examplecontinued}

\section{Plausible Information Aggregation}\label{sec:aggregation}
Repeated reports can also be combined to improve an answer, as in self-consistency methods that aggregate sampled reasoning paths by their final answers \citep{WangEtAl2023}. We ask which aggregate posteriors are compatible with Bayesian updating when the experiment generating the reports is unknown. This connects optimal information extraction from one reasoning agent to robust forecast aggregation across Bayesian experts with an unknown information structure \citep{ArieliEtAl2018Robust}.

Formally, the observer has a finite history of $T$ reports $h_T=(\nu_1,\dots,\nu_T)$ from repeated instances of a prompt about a random variable $Y$. To define optimal Bayesian aggregation, consider the benchmark in which the observer knows the prior $\mu_0$ and the experiment $E$ associated with the prompt. Write $\Pr$ for probability under this prior and experiment. The optimal belief aggregation is then given by the posterior prescribed by Bayes' rule as
\begin{align}
\nu^T(h_T)(y)
:=
\Pr(Y(\omega)=y\mid h_T)
\quad
\forall\,y\in\Y.
\end{align}

For any given $h_T$, we say that $\eta\in\Delta(\Y)$ is a \emph{Bayes-plausible aggregate after observing $h_T$} if there exist a full-support prior $\mu_0$, a finite experiment $E$, and a true state $\omega\in\Omega$ such that $h_T$ has positive probability conditional on $\omega$ and $\nu^T(h_T)=\eta$. For a full-support prior $\mu_0$, we say that $\eta$ is a \emph{$\mu_0$-Bayes-plausible aggregate after observing $h_T$} if the same conditions hold with the prior fixed at $\mu_0$.

\Needspace{9\baselineskip}
\begin{theorem}[Bayes-Plausible Aggregation]
\label{prop:bayes_plausible_aggregation}
Suppose that $|\Omega_y|\geq2$ for every $y\in\Y$ and $h_T$ contains at least two distinct reports.\footnote{\autoref{lem:single_observed_report} in the Appendix characterizes the case of all reports being the same.} For any full-support $\mu_0\in\Delta(\Omega)$, $\eta\in\Delta(\Y)$ is a $\mu_0$-Bayes-plausible aggregate after observing $h_T$ if and only if
\begin{align}
\supp\eta\subseteq\supp\nu_i
\qquad
\forall\,i=1,\dots,T.
\label{eq:aggregation_support}
\end{align}
\end{theorem}
\begin{compactproof}
See \appref{app:proof:bayes_plausible_aggregation}.
\end{compactproof}
 
As in the previous sections, if the observer does not know the experiment, the only guaranteed inference comes from hard evidence: any answer value $y\in\Y$ excluded by some report must also be excluded by the aggregate. Also, because condition \eqref{eq:aggregation_support} does not depend on $\mu_0$, it also characterizes Bayes-plausible aggregates when the prior is unknown.

\begin{remark}[Aggregation Bottleneck]
This result contrasts with the familiar formula for aggregation with a known prior when the agent reports beliefs about the full state (see \autoref{lem:full_belief_aggregation} in the appendix). In that setting, each report is a posterior $\mu$ over the complete state $\omega$, repetitions are conditionally i.i.d. given that state, and the Bayesian aggregate is a prior-corrected sum of reported log odds: for any $\omega,\omega'\in\supp\nu^T(h_T)$,
\begin{align}
\log\frac{\nu^T(h_T)(\omega)}{\nu^T(h_T)(\omega')}
=
\sum_{i=1}^{T}\log\frac{\nu_i(\omega)}{\nu_i(\omega')}
-(T-1)\log\frac{\mu_0(\omega)}{\mu_0(\omega')}.
\end{align}
Importantly, the formula highlights that the experiment \emph{need not be known} for optimal aggregation. Thus, it is the inability to communicate full posterior about the state of the world that creates a bottleneck for knowledge extraction. 

The reason is that repetitions are independent conditional on the complete state $\omega$, but generally not conditional on the coarser value $y=Y(\omega)$. The distribution is exchangeable in the repetitions, but the common latent state $\omega$ can correlate all reports after conditioning only on $y$. A full-state report identifies the state-by-state likelihood ratios needed in the product formula, whereas a coarse report identifies only their prior-weighted averages within each cell $\Omega_y$. Those averages do not determine the likelihood of a history without knowing the experiment, which is why coarse-report aggregation is much less restrictive.
\end{remark}

Although \autoref{prop:bayes_plausible_aggregation} permits many aggregates, it excludes some natural rules. Arithmetic averaging is not Bayes-plausible whenever one report assigns probability zero to an answer and another assigns it positive probability. Affine log-odds aggregation respects this support restriction and is exactly optimal in a Gaussian model (see \appref{app:affine_log_odds}).

\begin{examplecontinued}[ctd.]
We now show that even unlimited repetition can leave the Bayesian aggregate undetermined. Consider agent 2's report distribution in \autoref{fig:running_example}. Split the previous state $0$ into two and relabel the resulting space as $\Omega=\{0,1,2,3\}$, with $Y(0)=Y(1)=0$ and $Y(2)=Y(3)=1$. Let the prior be $\mu_0=(1,99,1,19)/120$. \appref{app:running_example_aggregation} constructs two fixed experiments with this prior and the same entire observed report distribution $\tau_2$. In one rationalization state $2$ is true and the Bayesian aggregate converges almost surely to $1$; in the other state $0$ is true and it converges almost surely to $0$. Each experiment is fixed for all $T$, and both generate the same distribution over every finite report history. The empirical mean converges to $1/2$ in both cases, while the Bayesian aggregates converge to opposite answers.

By contrast, replacing the lower report by an exact zero forces the aggregate to assign zero probability to $Y=1$ as soon as that report occurs. Thus, for purposes of aggregation, there is a substantive distinction between an arbitrarily small positive probability and zero.
\end{examplecontinued}

\section{Future Directions}\label{sec:discussion}
Our framework provides a theoretical basis for interpreting existing ML findings and developing more specialized models of reasoning. One direction is to extend our aggregation results by imposing additional structure on the internal experiments. Relevant applications include self-consistency, which selects the most frequent answer across reasoning attempts \citep{WangEtAl2023}, and semantic uncertainty, which treats answers with the same meaning as one outcome when measuring uncertainty \citep{KuhnGalFarquhar2023}. Such models could clarify when agreement across sampled answers predicts correctness and how those answers should be combined.

A second direction is to ask what checking or revising an answer adds. \citet{SnellEtAl2025} study verifier-guided search and answer revision, finding that effective computation allocation depends on question difficulty. Our information order could help assess when these procedures leave the agent better informed. If each step retains all earlier signals, additional signals cannot reduce Blackwell informativeness, but informativeness can decrease if only a summary of earlier reasoning is retained.

A third direction is to elicit more information about the agent's internal experiment. Direct elicitation may be problematic, as the agent may not understand its own inner workings. Instead, simpler, indirect tests could build on existing tests of Bayesian updating \citep{FalckWangHolmes2024,KimKimThorne2025}. For example, one could ask for a belief and the agent's probabilities for possible evidence, then supply new evidence and ask for an updated belief.

Our results provide a theoretical foundation for these directions and highlight that rigorous conclusions require explicit assumptions about how agents reason and report.
\label{iclr:main-text-end}

\medskip
\noindent \textbf{Acknowledgments:} Bryan Wilder gratefully acknowledges the following support for his work. Research was sponsored by the Army Research Office and was accomplished under Grant Number W911NF-25-1-0281. The views and conclusions contained in this document are those of the authors and should not be interpreted as representing the official policies, either expressed or implied, of the Army Research Office or the U.S. Government. The U.S. Government is authorized to reproduce and distribute reprints for Government purposes notwithstanding any copyright notation herein. This material is also based upon work supported by the AI Research Institutes Program funded by the National Science
Foundation under AI Institute for Societal Decision Making
(AI-SDM), Award No. 2229881.

Alex Smolin gratefully acknowledges funding from the French National Research Agency (ANR) under the Investments for the Future program (grant ANR-17-EURE-0010) and the AI Interdisciplinary Institute ANITI (grant ANR-23-IACL-0002).

\newpage
\appendix
\section{Proofs}
\label{app:proofs}

\subsection{Reduction to Reports on a Partition}
\label{app:partition_reduction}
We can interpret an arbitrary prompt as requesting the value of a function of the agent's posterior, $f_p:\Del(\Omega)\to R$. For a partition $\mathcal Q_p=\{Q_1,\ldots,Q_K\}$ of $\Omega$, define
\begin{align*}
A_p(\mu)(k):=\sum_{\omega\in Q_k}\mu(\omega),
\qquad k=1,\ldots,K.
\end{align*}
If $f_p$ is constant on each fiber of $A_p$, it factors as $f_p=g_p\circ A_p$ for some $g_p:\Del(\{1,\ldots,K\})\to R$. Such a partition always exists: the singleton partition makes $A_p$ a bijection. Set $Y_p(\omega)=k$ on $Q_k$. Then $\nu=A_p(\mu)$ is the posterior of $Y_p$, and the original response is $r=g_p(\nu)$, with no change to the experiment.

\subsection{Known Prior and Full-State Reports}
\label{app:proof:known_prior_full_state}
The next lemma characterizes the restrictions a known prior imposes on full-state reports. It also provides the implementation step in the proof of \autoref{thm:general_full_belief}: we lift the observed reports to full-state posteriors and choose a common prior satisfying the lemma's conditions for every prompt.

Throughout this subsection, set $\Y=\Omega$ and let $Y$ be the identity map. Fix a full-support prior $\mu_0$ and a candidate true state $\omega$. A finite-support distribution $\tau$ of full posteriors is \emph{$(\mu_0,\omega)$-Bayes-plausible} if some finite experiment induces $\tau$ conditional on $\omega$.

Whenever every $\nu\in\supp\tau$ satisfies $\nu(\omega)>0$, define
\begin{align}
H(\tau,\omega;z)
:=
\sum_{\nu\in\supp\tau}
\tau(\nu)\frac{\nu(z)}{\nu(\omega)}.
\label{eq:known_prior_odds_moment}
\end{align}

\begin{lemma}[Known-Prior Bayes-Plausibility]
\label{thm:known_prior_full_state}
Fix $\mu_0\in\operatorname{int}\Delta(\Omega)$ and $\omega\in\Omega$. A finite-support full-posterior report distribution $\tau$ is $(\mu_0,\omega)$-Bayes-plausible if and only if
\begin{align}
\nu(\omega)&>0
&&\text{for every }\nu\in\supp\tau,
\label{eq:known_prior_truth_support}\\
H(\tau,\omega;z)&\leq
\frac{\mu_0(z)}{\mu_0(\omega)}
&&\text{for every }z\in\Omega.
\label{eq:known_prior_moment_inequalities}
\end{align}
The inequality for $z=\omega$ always holds with equality.
\end{lemma}

\begin{proof}
Let $M:=\supp\tau$. Suppose first that a finite experiment $(S,\pi)$ induces $\tau$ at $\omega$. Delete ex ante null signals, let $q(s):=\sum_x\mu_0(x)\pi(s\mid x)$, and let $\nu_s$ be the posterior induced by $s$. If $\pi(s\mid\omega)>0$, then
\begin{align*}
\nu_s(\omega)
=
\frac{\mu_0(\omega)\pi(s\mid\omega)}{q(s)}
>0.
\end{align*}
Every report in $M$ is induced by at least one such signal, which proves \eqref{eq:known_prior_truth_support}.

For every $z\in\Omega$, Bayes' rule gives
\begin{align*}
H(\tau,\omega;z)
&=
\sum_{s:\,\pi(s\mid\omega)>0}
\pi(s\mid\omega)
\frac{\nu_s(z)}{\nu_s(\omega)}
=
\frac{\mu_0(z)}{\mu_0(\omega)}
\sum_{s:\,\pi(s\mid\omega)>0}\pi(s\mid z)
\leq
\frac{\mu_0(z)}{\mu_0(\omega)}.
\end{align*}
If several signals induce the same report, their probabilities at $\omega$ add to $\tau(\nu)$, so the first equality is valid. This proves necessity.

Conversely, suppose \eqref{eq:known_prior_truth_support} and \eqref{eq:known_prior_moment_inequalities} hold. Introduce one signal $s_\nu$ for each $\nu\in M$ and set
\begin{align}
\pi(s_\nu\mid z)
:=
\frac{\mu_0(\omega)}{\mu_0(z)}
\tau(\nu)
\frac{\nu(z)}{\nu(\omega)}.
\label{eq:known_prior_construct_likelihood}
\end{align}
Then $\pi(s_\nu\mid\omega)=\tau(\nu)$ and, for every $z$,
\begin{align*}
\sum_{\nu\in M}\pi(s_\nu\mid z)
=
\frac{\mu_0(\omega)}{\mu_0(z)}H(\tau,\omega;z)
\leq1.
\end{align*}
Add a residual signal $r$ with
\begin{align*}
\pi(r\mid z)
:=
1-\sum_{\nu\in M}\pi(s_\nu\mid z),
\end{align*}
omitting it if its likelihood is identically zero. Since $H(\tau,\omega;\omega)=1$, the residual signal satisfies $\pi(r\mid\omega)=0$.

The ex ante probability of $s_\nu$ is
\begin{align*}
q(s_\nu)
&=
\sum_{z\in\Omega}\mu_0(z)\pi(s_\nu\mid z)
=
\frac{\mu_0(\omega)\tau(\nu)}{\nu(\omega)}.
\end{align*}
Therefore
\begin{align*}
\Pr(z\mid s_\nu)
=
\frac{\mu_0(z)\pi(s_\nu\mid z)}{q(s_\nu)}
=
\nu(z).
\end{align*}
Thus $s_\nu$ induces $\nu$, occurs at $\omega$ with probability $\tau(\nu)$, and the residual signal never occurs at $\omega$. The induced report distribution is exactly $\tau$.
\end{proof}

The proof also shows that every rationalizing experiment satisfies
\begin{align}
\frac{\mu_0(\omega)}{\mu_0(z)}H(\tau,\omega;z)
&=
\sum_{s:\,\pi(s\mid\omega)>0}\pi(s\mid z)
=
1-\sum_{s:\,\pi(s\mid\omega)=0}\pi(s\mid z).
\label{eq:known_prior_missing_signal_mass}
\end{align}
Hence \eqref{eq:known_prior_moment_inequalities} holds with equality for every $z$ if and only if no signal possible at any state is impossible at $\omega$. In particular, equality holds when the experiment has common signal support across states.

For a fixed candidate truth, the identified set of full-support priors is
\begin{align}
\left\{
\mu_0\in\operatorname{int}\Delta(\Omega):
\mu_0(z)\geq\mu_0(\omega)H(\tau,\omega;z)
\ \text{for every }z\in\Omega
\right\}.
\label{eq:known_prior_identified_set}
\end{align}
Summing the inequalities gives
\begin{align}
\mu_0(\omega)
&\leq
\left(\sum_{z\in\Omega}H(\tau,\omega;z)\right)^{-1}
=
\left(
\EE_{\nu\sim\tau}\left[\frac{1}{\nu(\omega)}\right]
\right)^{-1}.
\label{eq:known_prior_harmonic_bound}
\end{align}
Jensen's inequality then implies $\EE_\tau[\nu(\omega)]\geq\mu_0(\omega)$. This is the \emph{truth-drifting} property presented by \citet{FrancetichKreps2014} and \citet{DovalSmolin2024}. The inverse-probability restriction is sharper than this mean restriction. In the binary case it is the full restriction. If $\Omega=\{0,1\}$, $\mu_0(1)=q$, the true state is $1$, and a report assigns probability $r$ to state $1$, Bayes-plausibility is equivalent to
\begin{align}
\EE_{r\sim\tau}\left[\frac{1-r}{r}\right]
\leq
\frac{1-q}{q},
\qquad\text{equivalently}\qquad
\EE_{r\sim\tau}\left[\frac{1}{r}\right]\leq\frac{1}{q}.
\label{eq:known_prior_binary}
\end{align}
If the true state is not specified, $\tau$ is Bayes-plausible under $\mu_0$ if and only if some $\omega\in\bigcap_{\nu\in\supp\tau}\supp\nu$ satisfies \eqref{eq:known_prior_moment_inequalities}. If the prior is also unrestricted, choosing $\mu_0(\omega)$ sufficiently small makes all inequalities hold. Thus the lemma reduces to the single-prompt full-state case of \autoref{thm:general_full_belief} when both the prior and the true state may vary.

\subsection{Proof of \texorpdfstring{\autoref{thm:general_full_belief}}{Bayes-Plausibility theorem}}
\label{app:proof:general_full_belief}
\begin{proof}
Suppose first that the collection is Bayes-plausible, and let $\omega$ be the common true state. For every $p\in P$ and $\nu\in\supp\tau_p$, some signal inducing $\nu$ occurs at $\omega$. Full support of the prior and Bayes' rule imply that the corresponding full posterior assigns positive probability to $\omega$. Its $Y_p$-marginal therefore satisfies $\nu(Y_p(\omega))>0$. Hence $\omega$ belongs to the intersection in \eqref{eq:general_compatibility_multiple}.

Conversely, suppose \eqref{eq:general_compatibility_multiple} holds and choose $\omega$ in that intersection. For each $p$, write $\Omega_p(y):=Y_p^{-1}(y)$ and lift each report $\nu\in\supp\tau_p$ by distributing its mass uniformly within each cell:
\begin{align}
L_p(\nu)(\omega')
:=
\frac{\nu(Y_p(\omega'))}{|\Omega_p(Y_p(\omega'))|}.
\label{eq:multiple_prompt_full_posterior_lift}
\end{align}
The posterior $L_p(\nu)$ has $Y_p$-marginal $\nu$ and assigns positive probability to $\omega$. Let $\widehat\tau_p$ be its distribution when $\nu\sim\tau_p$, and write $H_p(z):=H(\widehat\tau_p,\omega;z)$. These odds moments are finite and satisfy $H_p(\omega)=1$.

Choose a common prior by normalizing the positive weights
\begin{align*}
w(z):=
\begin{cases}
1,
&z=\omega,\\
1+\sum_{p\in P}H_p(z),
&z\neq\omega,
\end{cases}
\qquad
\mu_0(z):=\frac{w(z)}{\sum_{x\in\Omega}w(x)}.
\end{align*}
Then, for every $p\in P$ and $z\in\Omega$,
\begin{align}
H_p(z)
\leq
\frac{\mu_0(z)}{\mu_0(\omega)}.
\label{eq:multiple_prompt_prior_bound}
\end{align}
By \autoref{thm:known_prior_full_state}, each $\widehat\tau_p$ is induced at $\omega$ under this prior by a finite experiment. Reporting its $Y_p$-marginal induces $\tau_p$, proving sufficiency, including when $|\Omega|=1$.
\end{proof}

\subsection{Proof of \texorpdfstring{\autoref{prop:blackwell}}{Intelligence Order theorem}}
\label{app:proof:blackwell}
For the intelligence-order results below, let
\begin{align*}
C:=\bigcap_{\nu\in M_1\cup M_2}\supp\nu.
\end{align*}

\begin{proof}
\emph{Necessity.}
Delete ex ante null signals. Let $\omega$ be the true state, set $y:=Y(\omega)$, and let $q_j(s_j)$ and $\nu_j(s_j)$ be the ex ante probability and posterior report associated with signal $s_j$. Bayes' rule gives
\begin{align}
q_j(s_j)\nu_j(s_j)(z)
=
\sum_{\omega'\in\Omega_z}
\mu_0(\omega')\pi_j(s_j\mid\omega')
\qquad
\text{for every }z\in\Y.
\label{eq:blackwell_cell_bayes}
\end{align}
Every signal occurring at $\omega$ assigns positive probability to $y$, so $y\in C$.

Let $G$ garble $E_2$ into $E_1$. Conditional on $\omega$, draw $s_2$ and then $s_1$ from $G(\cdot\mid s_2)$. The induced distribution $\gamma$ of $(\nu_2(s_2),\nu_1(s_1))$ is a coupling of $\tau_2$ and $\tau_1$. Moreover,
\begin{align}
q_1(s_1)\nu_1(s_1)
=
\sum_{s_2\in S_2}
q_2(s_2)G(s_1\mid s_2)\nu_2(s_2).
\label{eq:blackwell_signal_barycenter}
\end{align}
If $(\nu_2,\nu_1)$ has positive $\gamma$-probability, some pair $(s_2,s_1)$ with those reports satisfies $\pi_2(s_2\mid\omega)G(s_1\mid s_2)>0$. If $\nu_1(z)=0$, equation \eqref{eq:blackwell_signal_barycenter} and nonnegativity force $\nu_2(z)=0$. Hence $\supp\nu_2\subseteq\supp\nu_1$ on the support of $\gamma$.

\emph{Sufficiency.}
Fix $y\in C$, a coupling $\gamma$ satisfying \eqref{eq:blackwell_support_coupling}, and a state $\omega\in\Omega_y$. We first give an implementation step.

Consider finitely many abstract signals $s$. Assign each signal an ex ante probability $q(s)>0$, a probability $a(s)\geq0$ at $\omega$, and a report $v_s\in\Del(\Y)$. Suppose
\begin{align*}
\sum_s q(s)=\sum_s a(s)=1,
\qquad
m:=\sum_s q(s)v_s\in\operatorname{int}\Del(\Y),
\end{align*}
and there is $x\in(0,m(y))$ such that
\begin{align*}
q(s)v_s(y)\geq xa(s)
\qquad
\text{for every }s.
\end{align*}
Choose a full-support prior with cell masses $\mu_0(\Omega_z)=m(z)$, assign mass $x$ to $\omega$, and distribute the remaining mass $m(y)-x$ strictly positively over $\Omega_y\setminus\{\omega\}$. This is possible because $|\Omega_y|\geq2$. Define
\begin{align*}
\pi(s\mid\omega')
:=
\begin{cases}
a(s),
&\omega'=\omega,\\
q(s)v_s(z)/m(z),
&\omega'\in\Omega_z,\ z\neq y,\\
\bigl(q(s)v_s(y)-xa(s)\bigr)/(m(y)-x),
&\omega'\in\Omega_y\setminus\{\omega\}.
\end{cases}
\end{align*}
The likelihoods are nonnegative and sum to one at every state. Signal $s$ has ex ante probability $q(s)$ and induces report $v_s$.

We now construct the abstract signals for $E_2$. Fix $c:=1/2$. By the support condition, one can choose $\varepsilon\in(0,c)$ so small that, for every $\nu_1\in M_1$,
\begin{align*}
w(\nu_1)
:=
c\tau_1(\nu_1)\nu_1
-
\varepsilon
\sum_{\nu_2\in M_2}\gamma(\nu_2,\nu_1)\nu_2
\end{align*}
is nonnegative and strictly positive on $\supp\nu_1$. A nonzero vector $w\in\RR_+^{\Y}$ represents a signal with ex ante probability $\lVert w\rVert_1$ and report $w/\lVert w\rVert_1$. Since $\gamma$ has second marginal $\tau_1$,
\begin{align*}
\lVert w(\nu_1)\rVert_1=(c-\varepsilon)\tau_1(\nu_1).
\end{align*}

For every pair with $\gamma(\nu_2,\nu_1)>0$, create a signal with report $\nu_2$, ex ante probability $\varepsilon\gamma(\nu_2,\nu_1)$, and probability $\gamma(\nu_2,\nu_1)$ at $\omega$. Garble it to a destination indexed by $\nu_1$. For each $\nu_1$, add the signal represented by $w(\nu_1)$, give it probability zero at $\omega$, and garble it to the same destination. Finally, add a signal with an interior report, ex ante probability $1-c$, probability zero at $\omega$, and a separate destination.

The ex ante probabilities sum to
\begin{align*}
\varepsilon+(c-\varepsilon)+(1-c)=1,
\end{align*}
and the probabilities at $\omega$ sum to one. The mean report is interior because the last signal has positive ex ante probability and an interior report. Choose
\begin{align*}
0<x<\varepsilon\min_{\nu_2\in M_2}\nu_2(y).
\end{align*}
For a signal indexed by $(\nu_2,\nu_1)$, the implementation inequality is
\begin{align*}
\varepsilon\gamma(\nu_2,\nu_1)\nu_2(y)
\geq
x\gamma(\nu_2,\nu_1).
\end{align*}
It is automatic for all signals with probability zero at $\omega$. Moreover, the pair signals alone give the mean report more than $x$ at $y$. The implementation step therefore produces $E_2$.

At the destination indexed by $\nu_1$, the prior-weighted report vector is
\begin{align*}
\varepsilon
\sum_{\nu_2\in M_2}\gamma(\nu_2,\nu_1)\nu_2
+w(\nu_1)
=
c\tau_1(\nu_1)\nu_1,
\end{align*}
while its probability at $\omega$ is $\sum_{\nu_2}\gamma(\nu_2,\nu_1)=\tau_1(\nu_1)$. Thus the garbled report is $\nu_1$. The ungarbled report distribution at $\omega$ is $\tau_2$, the garbled report distribution is $\tau_1$, and $E_2$ Blackwell dominates $E_1$.
\end{proof}

\subsection{Proof of the Different-Priors Remark}
\label{app:proof:different_priors}
Modify the definition of plausible Blackwell dominance by allowing agent $i$ to have its own full-support prior $\mu_i$ while retaining a common true state. If $|\Omega_y|\geq2$ for every $y\in\Y$, conditions \eqref{eq:blackwell_common_truth} and \eqref{eq:blackwell_support_coupling} continue to characterize plausible Blackwell dominance.
\begin{proof}
For each ex ante non-null signal $s_i$, write $\nu_i(s_i)$ for agent $i$'s report. Bayes' rule and full support of $\mu_i$ give
\begin{align}
z\in\supp\nu_i(s_i)
\quad\Longleftrightarrow\quad
\pi_i(s_i\mid\omega')>0
\text{ for some }\omega'\in\Omega_z.
\label{eq:different_prior_report_support}
\end{align}
Thus posterior support depends only on the experiment's likelihoods, not on
the positive weights in the prior.

For necessity, let $\omega$ be the common true state and let $G$ garble $E_2$ into $E_1$. As in the necessity proof of \autoref{prop:blackwell}, $Y(\omega)\in C$, and drawing $s_2$ at $\omega$ followed by $s_1$ from $G$ induces a coupling $\gamma$ of the report distributions. This construction does not require a common prior. Every pair in the support of $\gamma$ is induced by a signal pair with $\pi_2(s_2\mid\omega)G(s_1\mid s_2)>0$. For any
$z\in\supp\nu_2(s_2)$, \eqref{eq:different_prior_report_support} gives an
$\omega'\in\Omega_z$ such that $\pi_2(s_2\mid\omega')>0$. Therefore
\begin{align*}
\pi_1(s_1\mid\omega')
&=
\sum_{t\in S_2}\pi_2(t\mid\omega')G(s_1\mid t)
\geq
\pi_2(s_2\mid\omega')G(s_1\mid s_2)>0.
\end{align*}
Another application of \eqref{eq:different_prior_report_support} yields
$z\in\supp\nu_1(s_1)$, proving the required support inclusion.

Conversely, if the two conditions hold, \autoref{prop:blackwell} supplies a
rationalization with a common full-support prior. This is also admissible when
the agents are allowed, but not required, to have different priors.
\end{proof}

\subsection{Proof of \texorpdfstring{\autoref{cor:blackwell_binary}}{binary-event corollary}}
\label{app:proof:blackwell_binary}
\begin{proof}
Let $\nu^0$ and $\nu^1$ be the two extreme beliefs. Every other report has support $\Y$, so the common-truth condition in \autoref{prop:blackwell} holds exactly when $M_1\cup M_2$ contains at most one of $\nu^0$ and $\nu^1$. Under a support-respecting coupling, a less informative extreme report $\nu^k$ can only be paired with the same more informative report, implying $\tau_2(\nu^k)\geq\tau_1(\nu^k)$. Conversely, if these inequalities hold, match the mass $\tau_1(\nu^k)$ at each extreme belief and couple all residual mass to the nonextreme reports under $\tau_1$; every such pair satisfies support inclusion. The result follows from \autoref{prop:blackwell}.
\end{proof}

\begin{lemma}[Intelligence Order for $|\Omega_y|=1$]
\label{lem:blackwell_singleton}
Fix $y\in\Y$ such that $\Omega_y=\{\omega\}$. There exist a full-support prior and experiments $E_1,E_2$ such that $E_2$ Blackwell dominates $E_1$ and the report distribution under $E_i$, conditional on $\omega$, is $\tau_i$ for $i\in\{1,2\}$ if and only if $y\in C$ and there is a coupling $\gamma$ of $\tau_2$ and $\tau_1$ satisfying, coordinatewise,
\begin{align}
\sum_{\nu_2\in M_2}
\gamma(\nu_2,\nu_1)
\frac{\nu_2}{\nu_2(y)}
\leq
\tau_1(\nu_1)
\frac{\nu_1}{\nu_1(y)}
\qquad
\text{for every }\nu_1\in M_1.
\label{eq:blackwell_singleton_inequality}
\end{align}
This inequality implies $\supp\nu_2\subseteq\supp\nu_1$ whenever $\gamma(\nu_2,\nu_1)>0$.
\end{lemma}

\begin{proof}
The necessity argument for \autoref{prop:blackwell} gives $y\in C$ and a signal coupling. Let $t:=\mu_0(\omega)$. Since $\Omega_y=\{\omega\}$, \eqref{eq:blackwell_cell_bayes} gives
\begin{align*}
q_j(s_j)\nu_j(s_j)(y)
=
t\pi_j(s_j\mid\omega).
\end{align*}
Fix $\nu_1\in M_1$, sum \eqref{eq:blackwell_signal_barycenter} over all signals inducing $\nu_1$, and retain only signal pairs that occur at $\omega$. All omitted terms are nonnegative, so
\begin{align*}
t\tau_1(\nu_1)\frac{\nu_1}{\nu_1(y)}
\geq
t\sum_{\nu_2\in M_2}
\gamma(\nu_2,\nu_1)
\frac{\nu_2}{\nu_2(y)}.
\end{align*}
Canceling $t>0$ proves \eqref{eq:blackwell_singleton_inequality}. If $\nu_1(z)=0$, this condition and nonnegativity imply $\nu_2(z)=0$ for every $\nu_2$ paired with $\nu_1$, proving the stated support implication.

For sufficiency, the case $|\Y|=1$ is immediate. Suppose $|\Y|\geq2$ and choose $x>0$ such that
\begin{align*}
x\sum_{\nu_1\in M_1}\frac{\tau_1(\nu_1)}{\nu_1(y)}<1.
\end{align*}
For every pair with $\gamma(\nu_2,\nu_1)>0$, create a signal with report $\nu_2$, probability $\gamma(\nu_2,\nu_1)$ at $\omega$, and ex ante probability
\begin{align*}
q(\nu_2,\nu_1)
:=
x\frac{\gamma(\nu_2,\nu_1)}{\nu_2(y)}.
\end{align*}
Garble it to a destination indexed by $\nu_1$. Define the residual vector
\begin{align*}
w(\nu_1)
:=
x\left(
\tau_1(\nu_1)\frac{\nu_1}{\nu_1(y)}
-
\sum_{\nu_2\in M_2}
\gamma(\nu_2,\nu_1)
\frac{\nu_2}{\nu_2(y)}
\right).
\end{align*}
It is nonnegative by \eqref{eq:blackwell_singleton_inequality}, and its $y$-coordinate is zero. If $w(\nu_1)\neq0$, add the signal represented by this vector to the same destination and give it probability zero at $\omega$. Summing the pair masses and the $\ell_1$-norms of the residual vectors gives total ex ante probability
\begin{align*}
x\sum_{\nu_1\in M_1}\frac{\tau_1(\nu_1)}{\nu_1(y)}<1.
\end{align*}
Use the remaining probability for a signal that has probability zero at $\omega$, assigns zero probability to $y$, assigns positive probability to every other value, and has a separate destination.

Let $q(s)$, $a(s)$, and $v_s$ denote the ex ante probability, probability at $\omega$, and report of each signal. Their ex ante probabilities and probabilities at $\omega$ both sum to one. Their mean report $m:=\sum_s q(s)v_s$ is interior, $m(y)=x$, and
\begin{align*}
q(s)v_s(y)=xa(s)
\qquad
\text{for every }s.
\end{align*}
Give each cell $\Omega_z$ prior mass $m(z)$, distributing it strictly positively within the cell. In particular, $\mu_0(\omega)=x$. Define
\begin{align*}
\pi_2(s\mid\omega')
:=
\begin{cases}
a(s),
&\omega'=\omega,\\
q(s)v_s(z)/m(z),
&\omega'\in\Omega_z,\ z\neq y.
\end{cases}
\end{align*}
These likelihoods sum to one at every state. Since $q(s)v_s(y)=xa(s)$, signal $s$ has ex ante probability $q(s)$ and induces report $v_s$.

At the destination indexed by $\nu_1$, the prior-weighted report vector is $x\tau_1(\nu_1)\nu_1/\nu_1(y)$ and the probability at $\omega$ is $\tau_1(\nu_1)$. Hence its report is $\nu_1$. The ungarbled and garbled report distributions at $\omega$ are $\tau_2$ and $\tau_1$.
\end{proof}

\subsection{Full-Posterior Aggregation}
\label{app:full_belief_aggregation}
\begin{lemma}[Full-Posterior Aggregation]
\label{lem:full_belief_aggregation}
Suppose that the prompt asks the agent to report its full posterior over the state of the world, so that each $\nu_i\in\Delta(\Omega)$. Then, for every history $h_T$ that occurs with positive probability,
\begin{align}
\nu^T(h_T)(\omega)
=
\frac{
\mu_0(\omega)^{1-T}
\prod_{i=1}^{T}\nu_i(\omega)
}{
\sum_{\omega'\in\Omega}
\mu_0(\omega')^{1-T}
\prod_{i=1}^{T}\nu_i(\omega')
}
\qquad
\text{for every }\omega\in\Omega.
\label{eq:full_belief_aggregation}
\end{align}
In particular, for any $\omega,\omega'\in\supp\nu^T(h_T)$,
\begin{align}
\log\frac{\nu^T(h_T)(\omega)}{\nu^T(h_T)(\omega')}
=
\sum_{i=1}^{T}\log\frac{\nu_i(\omega)}{\nu_i(\omega')}
-(T-1)\log\frac{\mu_0(\omega)}{\mu_0(\omega')}.
\label{eq:full_belief_log_odds}
\end{align}
\end{lemma}

\begin{proof}
Let $q(\nu)$ be the ex ante probability of report $\nu$. Since the report is the posterior itself, Bayes' rule gives
\begin{align}
\Pr(\text{report}=\nu\mid\omega)
=
q(\nu)\frac{\nu(\omega)}{\mu_0(\omega)}.
\label{eq:report_likelihood_full_belief}
\end{align}
Conditional independence therefore implies
\begin{align*}
\Pr(\omega\mid h_T)
&\propto
\mu_0(\omega)
\prod_{i=1}^{T}
q(\nu_i)\frac{\nu_i(\omega)}{\mu_0(\omega)}
\propto
\mu_0(\omega)^{1-T}
\prod_{i=1}^{T}\nu_i(\omega).
\end{align*}
Normalizing proves \eqref{eq:full_belief_aggregation}; taking log odds proves \eqref{eq:full_belief_log_odds}.
\end{proof}

\subsection{Proof of \texorpdfstring{\autoref{prop:bayes_plausible_aggregation}}{Bayes-Plausible Aggregation theorem}}
\label{app:proof:bayes_plausible_aggregation}
\begin{proof}
Let $M:=\{\nu_i:i=1,\ldots,T\}$ be the set of distinct observed reports, and let $n(\nu)$ be the multiplicity of $\nu\in M$. For each $\nu\in M$, let
\begin{align*}
\mu^\nu(\omega)
:=
\Pr(\omega\mid\text{report}=\nu)
\end{align*}
be its full-state posterior lift. Its $Y$-marginal is the report:
\begin{align}
\sum_{\omega\in\Omega_y}\mu^\nu(\omega)
=
\nu(y)
\qquad
\text{for every }\nu\in M\text{ and }y\in\Y.
\label{eq:aggregation_lift_marginal}
\end{align}
If $q_\nu$ is the ex ante probability of report $\nu$, then
\begin{align*}
\Pr(\text{report}=\nu\mid\omega)
=
q_\nu\frac{\mu^\nu(\omega)}{\mu_0(\omega)}.
\end{align*}
Conditional independence gives
\begin{align}
\Pr(\omega\mid h_T)
\propto
\mu_0(\omega)^{1-T}
\prod_{\nu\in M}\mu^\nu(\omega)^{n(\nu)}.
\label{eq:aggregation_lift_formula}
\end{align}
Thus, if
\begin{align}
W_y
:=
\sum_{\omega\in\Omega_y}
\mu_0(\omega)^{1-T}
\prod_{\nu\in M}\mu^\nu(\omega)^{n(\nu)},
\label{eq:aggregation_cell_weight}
\end{align}
then the aggregate assigns probability $W_y/\sum_zW_z$ to $y$. If this probability is positive, some $\omega\in\Omega_y$ satisfies $\mu^\nu(\omega)>0$ for every $\nu\in M$. By \eqref{eq:aggregation_lift_marginal}, every observed report then assigns positive probability to $y$. This proves necessity.

We next record a fixed-prior implementation step. Consider any finite collection of full-state posteriors $(\mu^\nu)_{\nu\in M}$ satisfying \eqref{eq:aggregation_lift_marginal}. Choose numbers $p_\nu>0$ so small that
\begin{align*}
\sum_{\nu\in M}p_\nu<1,
\qquad
\sum_{\nu\in M}p_\nu\mu^\nu(\omega)<\mu_0(\omega)
\quad
\text{for every }\omega.
\end{align*}
Set
\begin{align*}
p_r
:=
1-\sum_{\nu\in M}p_\nu,
\qquad
\mu^r
:=
\frac{\mu_0-\sum_{\nu\in M}p_\nu\mu^\nu}{p_r}.
\end{align*}
Then $\mu^r$ has full support. The theorem's hypothesis of two distinct reports implies $|\Y|\geq2$. Choose states $a,b\in\Omega$ with $Y(a)\neq Y(b)$, and let $d:=\delta(a)-\delta(b)$. For all sufficiently small $\varepsilon>0$, the distributions
\begin{align*}
\mu^+:=\mu^r+\varepsilon d,
\qquad
\mu^-:=\mu^r-\varepsilon d
\end{align*}
have full support. By excluding finitely many values of $\varepsilon$, their $Y$-marginals can also be chosen outside the finite set $M$. Define
\begin{align}
\pi(s_\nu\mid\omega)
&:=
p_\nu\frac{\mu^\nu(\omega)}{\mu_0(\omega)},
&
\pi(r^\pm\mid\omega)
&:=
\frac{p_r}{2}
\frac{\mu^\pm(\omega)}{\mu_0(\omega)}.
\label{eq:aggregation_implementation}
\end{align}
These likelihoods sum to one at every state because
\begin{align*}
\sum_{\nu\in M}p_\nu\mu^\nu
+
\frac{p_r}{2}(\mu^++\mu^-)
=
\mu_0.
\end{align*}
Bayes' rule gives posterior $\mu^\nu$ after $s_\nu$, while the residual signals produce reports outside $M$. Thus any such collection of lifts is implementable under the fixed prior, with no additional signal producing an observed report.

For sufficiency, let $D:=\supp\eta\subseteq\bigcap_{\nu\in M}\supp\nu$. Choose a report $\beta\in M$ and distinct states $a_y,b_y\in\Omega_y$ for every $y$. For parameters $t_y\in[0,\beta(y)]$, define the lifts in every cell by

\begin{align*}
\mu^\nu(a_y)&:=\nu(y)
&&\text{for }\nu\neq\beta,
&
\mu^\beta(a_y)&:=t_y,
&
\mu^\beta(b_y)&:=\beta(y)-t_y.
\end{align*}
Set every unspecified coordinate to zero. These lifts satisfy \eqref{eq:aggregation_lift_marginal}. Since $M$ contains a report other than $\beta$, only $a_y$ can contribute to the cell weight $W_y$.

For $y\in D$, define
\begin{align*}
A_y
:=
\mu_0(a_y)^{1-T}
\prod_{\nu\in M\setminus\{\beta\}}\nu(y)^{n(\nu)}
>0.
\end{align*}
Choose $c>0$ small enough that
\begin{align*}
t_y
:=
\left(\frac{c\eta(y)}{A_y}\right)^{1/n(\beta)}
<
\beta(y)
\qquad
\text{for every }y\in D.
\end{align*}
Then \eqref{eq:aggregation_cell_weight} gives $W_y=A_yt_y^{n(\beta)}=c\eta(y)$ on $D$. Set $t_y=0$ off $D$, so that $W_y=0$ there. The resulting aggregate is $\eta$.

The fixed-prior implementation step realizes these lifts. For any $y\in D$, every observed report has positive conditional probability at the state $a_y$, so the history occurs with positive probability there. This proves sufficiency.
\end{proof}

\begin{lemma}[Bayes-Plausible Aggregation of Identical Reports]
\label{lem:single_observed_report}
Suppose $|\Omega_y|\geq2$ for every $y\in\Y$ and $h_T=(\nu,\ldots,\nu)$. For $T=1$, the unique Bayes-plausible aggregate is $\eta=\nu$, also under any fixed full-support prior.
For $T\geq2$, an aggregate $\eta$ is Bayes-plausible if and only if
\begin{align}
\supp\eta=\supp\nu.
\label{eq:single_report_support}
\end{align}
For a fixed full-support prior $\mu_0$, define
\begin{align}
L_y:=\frac{\nu(y)^T}{\mu_0(\Omega_y)^{T-1}},
\qquad
U_y:=\frac{\nu(y)^T}{\bigl(\min_{\omega\in\Omega_y}\mu_0(\omega)\bigr)^{T-1}}
\quad(y\in\supp\nu).
\label{eq:single_report_bounds}
\end{align}
Then $\eta$ is $\mu_0$-Bayes-plausible if and only if \eqref{eq:single_report_support} holds and some $c>0$ satisfies
\begin{align}
L_y\leq c\eta(y)\leq U_y
\quad\text{for all }y\in\supp\nu,
\label{eq:single_report_fixed_prior}
\end{align}
or, equivalently,
\begin{align}
\max_{y\in\supp\nu}\frac{L_y}{\eta(y)}
\leq
\min_{y\in\supp\nu}\frac{U_y}{\eta(y)}.
\label{eq:single_report_intersection}
\end{align}
\end{lemma}

\subsection{Proof of \texorpdfstring{\autoref{lem:single_observed_report}}{the single-observed-report lemma}}
\label{app:proof:single_observed_report}
\begin{proof}
The case $|\Y|=1$ is immediate. Otherwise, the construction in \eqref{eq:aggregation_implementation} implements any full-state lift $x$ of $\nu$ under $\mu_0$. Choosing any state in $\supp x$ as the true state makes the repeated report occur with positive probability. By \eqref{eq:aggregation_lift_formula}, its unnormalized aggregate weights are
\begin{align}
W_y(x)=\sum_{\omega\in\Omega_y}\frac{x(\omega)^T}{\mu_0(\omega)^{T-1}},
\qquad
\sum_{\omega\in\Omega_y}x(\omega)=\nu(y).
\label{eq:single_report_cell_weight}
\end{align}
For $T=1$, these weights equal $\nu(y)$. For $T\geq2$, they are positive exactly on $\supp\nu$. Within each cell, H\"older's inequality gives the minimum $L_y$, attained at $x(\omega)=\nu(y)\mu_0(\omega)/\mu_0(\Omega_y)$. Convexity gives the maximum $U_y$, attained by concentrating mass on a state minimizing $\mu_0$. Continuity on the connected cell simplex gives every intermediate value. Since cells can be treated independently, their normalized weights equal $\eta$ exactly when \eqref{eq:single_report_fixed_prior} holds, equivalently \eqref{eq:single_report_intersection}.

For unrestricted sufficiency when $T\geq2$, assume \eqref{eq:single_report_support}. Choose $a_y\in\Omega_y$ for each $y\in\supp\nu$, set $x(a_y)=\nu(y)$ and $x=0$ elsewhere, and choose a full-support prior with
\begin{align*}
\mu_0(a_y)=\kappa\left(\frac{\nu(y)^T}{\eta(y)}\right)^{1/(T-1)}.
\end{align*}
Take $\kappa>0$ small enough that these masses sum to less than one and distribute the remainder strictly positively over all other states, which exist because $|\Omega_y|\geq2$. Then $W_y(x)=\kappa^{1-T}\eta(y)$, proving sufficiency.
\end{proof}

\subsection{Affine Log-Odds Aggregation}
\label{app:affine_log_odds}

Fix a history $h_T=(\nu_1,\dots,\nu_T)$ of reports about $Y$.
Every Bayes-plausible aggregate must be supported on $D:=\bigcap_{i=1}^T\supp\nu_i$. If $D=\varnothing$, no Bayes-plausible aggregate exists. Otherwise, fix a base value $y_0\in D$ and enumerate the remaining values in $D$ as $\{y_1,\dots,y_K\}$. Define the report log-odds vectors and their average by
\begin{align*}
l_i(k)
&:=\log\frac{\nu_i(y_k)}{\nu_i(y_0)},
&
\bar l
&:=\frac1T\sum_{i=1}^T l_i,
&m(k):=\log \frac{\nu(y_k)}{\nu(y_0)},
&& i=1,\dots,T,\quad k=1,\dots,K.
\end{align*}
For parameters $c\in\reals^{K}$ and $C\in\reals^{K\times K}$, affine log-odds aggregation sets
\begin{align}
m
=
c+C\bar l.
\label{eq:weighted_log_odds}
\end{align}
The log-odds vector $m$ uniquely determines the aggregate $\nu$ with support exactly $D$. Under the hypotheses of \autoref{prop:bayes_plausible_aggregation}, the resulting $\nu$ is $\mu_0$-Bayes-plausible for every full-support prior $\mu_0$, and therefore Bayes-plausible. In \eqref{eq:weighted_log_odds}, roughly, $c$ captures the prior knowledge and $C$ captures the learning patterns across $y\in\Y$. The choice $c=0$ and $C=I_K$ gives ordinary average log odds, the equal-weight logarithmic opinion pool. This pool is externally Bayesian: pooling commutes with updating each input distribution by the same likelihood \citep{Genest1984}. The choice $c=(1-T)l_0$ and $C=TI_K$ corresponds to the standard Bayesian aggregation rule with a known prior log odds $l_0$ and conditionally independent reports. More generally, affine log-odds aggregation is Bayes-optimal in the Gaussian intelligence model presented in \appref{app:gaussian_reasoning}. In its homogenous case, $C=\lambda I_K$ and $c=(1-\lambda)l_0$, where $\lambda\in[1,T]$ captures the correlation across reports.

\subsection{Gaussian Intelligence}
\label{app:gaussian_reasoning}

Conditional independence in the model is imposed given the full state $\omega$, not generally given the coarser value $y=Y(\omega)$. Repeated reports may therefore remain correlated conditional on $y$ because they share components of $\omega$ not revealed by $y$. The following Gaussian specification makes this dependence explicit. Related generative Bayesian aggregators for independent and exchangeable experts are developed by \citet{Kahn2004}.

Allow the full state to have a continuous component and signals to have continuous realizations. Let $J:=|\Y|$ and, without loss, $\Y=\{1,\ldots,J\}$; set $K:=J-1$. For any full-support posterior $\nu\in\Delta(\Y)$, write its log odds relative to $J$ as
\begin{align*}
l(k):=\log\frac{\nu(k)}{\nu(J)},
\qquad
k=1,\ldots,K,
\end{align*}
and let $l_0(k):=\log(\Pr(y=k)/\Pr(y=J))$ be the prior log odds. Let $e_1,\ldots,e_K$ be the standard basis of $\RR^K$ and set $e_J:=0$. Fix symmetric matrices $G,\Sigma\in\RR^{K\times K}$ with $G\succ0$ and $0\preceq\Sigma\preceq G$.

Let the full state be $(y,u)$, where $u\sim N(0,\Sigma)$ independently of $y$. Conditional on $(y,u)$, repetition $i$ produces
\begin{align}
s_i
=
l_0-\frac12\operatorname{diag}(G)+Ge_y+u+\varepsilon_i,
\qquad
i=1,\ldots,T,
\label{eq:canonical_gaussian_model}
\end{align}
where $\varepsilon_i\sim N(0,G-\Sigma)$ independently across $i$ and independently of $(y,u)$. Thus the signals are i.i.d. conditional on the full state and exchangeable conditional on $y$.

To verify calibration, let
\begin{align*}
m_J:=l_0-\frac12\operatorname{diag}(G),
\qquad
m_k:=m_J+Ge_k.
\end{align*}
Then $s_i\mid y=k\sim N(m_k,G)$. If $f_k$ denotes its density, then, for $k=1,\ldots,K$,
\begin{align*}
\log\frac{f_k(s)}{f_J(s)}
&=
(m_k-m_J)^\top G^{-1}
\left(s-\frac{m_k+m_J}{2}\right)
=
s(k)-l_0(k).
\end{align*}
Adding prior log odds gives
\begin{align*}
\log\frac{\Pr(y=k\mid s_i)}{\Pr(y=J\mid s_i)}
=s_i(k).
\end{align*}
Hence the signal itself is the posterior log-odds vector, so write $l_i:=s_i$.

For aggregation, define
\begin{align*}
\bar l_T:=\frac1T\sum_{i=1}^T l_i,
\qquad
V_T:=\Sigma+\frac{G-\Sigma}{T}.
\end{align*}
Since $V_T=G/T+(1-1/T)\Sigma\succ0$, this covariance is invertible even when $\Sigma$ or $G-\Sigma$ is singular. Conditional on $y=k$, $\bar l_T\sim N(m_k,V_T)$. Moreover,
\begin{align*}
\operatorname{Cov}(\bar l_T,l_i-\bar l_T\mid y)=0.
\end{align*}
Conditional on $y$, the average and the centered residuals are jointly Gaussian and therefore independent. The distribution of the residuals does not depend on $y$, so $\bar l_T$ is sufficient for $y$.

\begin{proposition}[Gaussian aggregation]
\label{prop:canonical_gaussian_aggregation}
In the Gaussian intelligence model, there exist a vector $c_T(l_0,\Sigma,G)\in\RR^K$ and a matrix $C_T(\Sigma,G)\in\RR^{K\times K}$ such that the posterior log odds after observing $l_1,\ldots,l_T$ are
\begin{align}
l^T
=
c_T(l_0,\Sigma,G)
+
C_T(\Sigma,G)
\frac1T\sum_{i=1}^T l_i.
\label{eq:canonical_gaussian_aggregation}
\end{align}
\end{proposition}

\begin{proof}
Let $g_k$ be the conditional density of $\bar l_T$ given $y=k$. Since $m_k-m_J=Ge_k$,
\begin{align*}
\log\frac{g_k(\bar l_T)}{g_J(\bar l_T)}
&=
(m_k-m_J)^\top V_T^{-1}
\left(\bar l_T-\frac{m_k+m_J}{2}\right)\\
&=
e_k^\top GV_T^{-1}
\left(\bar l_T-l_0+\frac12\operatorname{diag}(G)\right)
-
\frac12e_k^\top GV_T^{-1}Ge_k.
\end{align*}
Adding the prior log odds and stacking the coordinates gives
\begin{align*}
l^T
=
l_0
+
GV_T^{-1}
\left(\bar l_T-l_0+\frac12\operatorname{diag}(G)\right)
-
\frac12\operatorname{diag}(GV_T^{-1}G).
\end{align*}
Thus \eqref{eq:canonical_gaussian_aggregation} holds with
\begin{align*}
C_T(\Sigma,G)
&:=GV_T^{-1},\\
c_T(l_0,\Sigma,G)
&:=l_0-GV_T^{-1}l_0
+\frac12GV_T^{-1}\operatorname{diag}(G)
-\frac12\operatorname{diag}(GV_T^{-1}G).
\end{align*}
\end{proof}

\begin{corollary}[Homogenous Gaussian aggregation]
\label{cor:scalar_gaussian_aggregation}
Suppose that $\Sigma=\rho G$ for some $\rho\in[0,1]$. Then
\begin{align*}
l^T=(1-\lambda_T)l_0+\lambda_T\bar l_T,
\qquad
\lambda_T:=\frac{T}{1+(T-1)\rho}.
\end{align*}
Conversely, if $T\geq2$, then $C_T(\Sigma,G)$ is a scalar multiple of the identity if and only if $\Sigma$ is proportional to $G$.
\end{corollary}

\begin{proof}
If $\Sigma=\rho G$, then
\begin{align*}
V_T
=
\frac{1+(T-1)\rho}{T}G,
\end{align*}
so $C_T=GV_T^{-1}=\lambda_T I_K$. Substitution into the expression for $c_T$ gives $c_T=(1-\lambda_T)l_0$.

Conversely, suppose $T\geq2$ and $GV_T^{-1}=\lambda I_K$. Since $G$ and $V_T$ are positive definite, $\lambda>0$ and $V_T=\lambda^{-1}G$. Using $TV_T=G+(T-1)\Sigma$ shows that $\Sigma$ is proportional to $G$.
\end{proof}

If $\Sigma=0$, the formula reduces to $l^T=T\bar l_T-(T-1)l_0$, the expression for conditionally independent reports. If $\Sigma=G$, all reports coincide and $l^T=\bar l_T=l_1$.

\section{Running Example}\label{app:running_example}

Throughout this example, $r$ denotes the reported probability of $Y=1$, and the two report distributions at the realized state are
\begin{align*}
\tau_1=\tfrac12\delta_{2/5}+\tfrac12\delta_{3/5},
\qquad
\tau_2=\tfrac12\delta_{1/10}+\tfrac12\delta_{9/10}.
\end{align*}

\subsection{Opposite Intelligence Orders}\label{app:running_example_order}

On a two-state space, neither Blackwell order is possible for these distributions, for any common full-support prior. At realized state $1$, \autoref{lem:blackwell_singleton} implies that $E_2\succeq_B E_1$ requires
\begin{align*}
\EE_{\tau_2}\left[\frac{1-r}{r}\right]
\leq
\EE_{\tau_1}\left[\frac{1-r}{r}\right],
\end{align*}
but the two sides are $41/9$ and $13/12$. Conversely, every ex ante possible full posterior under an experiment inducing $\tau_1$ at state $1$ lies in $\{0,2/5,3/5\}$: any signal absent at state $1$ has posterior $0$. A garbling cannot produce the report $9/10$, ruling out $E_1\succeq_B E_2$. Symmetry gives both impossibilities when the realized state is $0$.

Now let $\Omega=\{0,1,2\}$, with $Y(0)=0$ and $Y(1)=Y(2)=1$, realized state $1$, and common full-state prior
\begin{align*}
\mu_0=\tfrac1{120}(100,1,19).
\end{align*}
With rows ordered as $(0,1,2)$, consider the conditional signal matrices
\begin{align}
\pi_2=\begin{pmatrix}
9/200&1/200&3/20&4/5\\
1/2&1/2&0&0\\
0&4/19&15/19&0
\end{pmatrix},
\qquad
\pi_1=\begin{pmatrix}
3/25&2/25&4/5\\
1/2&1/2&0\\
15/38&23/38&0
\end{pmatrix}.
\label{eq:running_forward_order}
\end{align}
The experiment $E_2$ from \eqref{eq:running_three_state_experiment} induces reports $(1/10,9/10,1/2,0)$. Send its first two signals to separate destinations, split its third signal equally between those destinations, and send its fourth signal to a residual destination. This garbling gives $E_1$, whose reports are $(2/5,3/5,0)$. At state $1$, each experiment generates its first two reports with probability $1/2$, giving exactly the plotted distributions. The prior over $Y$ remains $(5/6,1/6)$. The order is strict: the ex ante squared-error Bayes risks for predicting $Y$ are $7/100$ under $E_2$ and $8/100$ under $E_1$.

The reverse strict order is possible at the same full-state prior and realized state. Consider experiments $\widehat E_1,\widehat E_2$ with conditional signal matrices
\begin{align}
\widehat\pi_1=\begin{pmatrix}
3/100&1/50&19/20&0\\
1/2&1/2&0&0\\
3/38&5/38&0&15/19
\end{pmatrix},
\qquad
\widehat\pi_2=\begin{pmatrix}
9/50&1/50&4/5\\
1/2&1/2&0\\
3/38&35/38&0
\end{pmatrix}.
\label{eq:running_reverse_order}
\end{align}
The reports under $\widehat E_1$ are $(2/5,3/5,0,1)$. Send its first two signals to separate destinations; send its third signal to the first destination with probability $3/19$ and to a residual destination otherwise; send its fourth signal to the second destination. The resulting experiment is $\widehat E_2$, with reports $(1/10,9/10,0)$. The realized report distributions are again $(\tau_1,\tau_2)$, while squared-error Bayes risks are now $2/100$ for $\widehat E_1$ and $3/100$ for $\widehat E_2$.

For the endpoint comparison, keep the same three-state space and allow the common prior and experiments to vary. Replacing $\tau_2$ by $\tfrac12\delta_{1/10}+\tfrac12\delta_1$ leaves $Y=1$ as the common possible answer, with two states in its cell. By the proof of \autoref{cor:blackwell_binary}, this distribution plausibly dominates $\tau_1$, whereas the reverse order is impossible. Replacing its other atom by $0$ leaves no common possible answer and violates \autoref{thm:general_full_belief}.

\subsection{Aggregation with Arbitrarily Many Reports}\label{app:running_example_aggregation}

To study unlimited repetition, we construct two fixed experiments with the same prior and observed report distribution but opposite limiting aggregates.

Let $\Omega=\{0,1,2,3\}$, with $Y(0)=Y(1)=0$ and $Y(2)=Y(3)=1$, and fix
\begin{align*}
\mu_0=\tfrac1{120}(1,99,1,19).
\end{align*}
Consider experiments $E^+,E^-$ with conditional signal matrices (rows ordered as $(0,1,2,3)$)
\begin{align}
\pi^+=\begin{pmatrix}
0&0&1\\
1/22&1/198&94/99\\
1/2&1/2&0\\
0&4/19&15/19
\end{pmatrix},
\qquad
\pi^-=\begin{pmatrix}
1/2&1/2&0\\
4/99&0&95/99\\
0&0&1\\
1/38&9/38&14/19
\end{pmatrix}.
\label{eq:running_aggregation_models}
\end{align}
At the stated prior, both experiments have ex ante signal probabilities $(1/24,1/24,11/12)$ and induce reports $(1/10,9/10,3/22)$. The realized state is $2$ under $E^+$ and $0$ under $E^-$. Each generates the entire report distribution $\tau_2$.

Let $n_L,n_H\geq1$ count the reports $1/10,9/10$, with $T=n_L+n_H$. The two aggregates are
\begin{align}
\eta_T^+
&=\frac1{1+99\,11^{-n_L}99^{-n_H}},
&
\eta_T^-
&=\frac{9^{n_H}}{19^{T-1}+9^{n_H}}.
\label{eq:running_finite_aggregates}
\end{align}
Under either realized distribution, both reports eventually appear almost surely. For every subsequent $T$,
\begin{align*}
0\leq1-\eta_T^+&\leq99\,11^{-T}\longrightarrow0,
&
0\leq\eta_T^-&\leq19(9/19)^T\longrightarrow0.
\end{align*}
Thus the Bayesian aggregate converges almost surely to $1$ under $E^+$ and to $0$ under $E^-$. The experiments are fixed for all $T$; their realized states differ.

Both aggregation models also reproduce $\tau_1$ for a second agent using the common garbling
\begin{align*}
G=\begin{pmatrix}
5/8&3/8&0\\
3/8&5/8&0\\
0&0&1
\end{pmatrix}.
\end{align*}
Its reports are $(2/5,3/5,3/22)$, with equal probabilities on the first two at the respective realized states. Thus both plotted report distributions, the common prior, and the ex ante report probabilities agree across the two aggregation models.

\end{document}